%% file: main.tex
\documentclass[dvipsnames,table,xcdraw]{article}

\include{prel}
\usepackage[accepted]{icml2021arxiv}

\icmltitlerunning{HardCoRe-NAS}

\usepackage{cleveref}[2012/02/15]
\crefformat{footnote}{#2\footnotemark[#1]#3}

\begin{document}
\twocolumn[
\icmltitle{HardCoRe-NAS: Hard Constrained diffeRentiable\\Neural Architecture Search}
\icmlsetsymbol{equal}{*}
\begin{icmlauthorlist}
\icmlauthor{Niv Nayman}{equal,to}
\icmlauthor{Yonathan Aflalo}{equal,to}
\icmlauthor{Asaf Noy}{to}
\icmlauthor{Lihi Zelnik-Manor}{to}
\end{icmlauthorlist}
\icmlaffiliation{to}{Alibaba Group, Tel Aviv, Israel}
\icmlcorrespondingauthor{Niv Nayman}{niv.nayman@alibaba-inc.com}
\icmlcorrespondingauthor{Yonathan Aflalo}{johnaflalo@gmail.com}
\icmlkeywords{Machine Learning, ICML}
\vskip 0.16in
]
\printAffiliationsAndNotice{\icmlEqualContribution} 
\input{abstract-concise}
\input{acc_vs_lat_graph}
\input{introduction}
\input{related}
\input{method}

\input{exp}
\input{conclusion}

\bibliography{references}
\bibliographystyle{icml2021}
\newpage
\input{supplementary}
\end{document}

%% file: prel.tex
\usepackage{graphicx}
\usepackage{graphics}
\usepackage{amsmath}
\usepackage{amsfonts}
\usepackage{amssymb}
\usepackage{amsthm}
\usepackage{bm}
\usepackage{pgfplots}
\usepackage{adjustbox}
\usepackage{enumerate}
\usepackage{textcomp}
\usepackage{tikz}
\usepackage{tabu}
\usepackage{caption}
\usetikzlibrary{shapes, arrows, decorations.text, shadows.blur, backgrounds, positioning,fit,calc}

\usepackage{array} 
\usepackage{longtable} 
\pgfplotsset{compat=1.16}
\usepackage{microtype}
\usepackage{multirow}
\usepackage{makecell} 
\usepackage{pifont}
\usepackage{booktabs} 
\usepackage{algorithmic}
\usepackage{algorithm}
\usepackage{amsmath,amssymb,amsthm}
\usepackage{mathtools}
\newtheorem{de}{Definition}[section]
\newtheorem{theo}{Theorem}[section]

\newtheorem{lem}[theo]{Lemma}
\newenvironment{remark}[1][Remark]{\begin{trivlist}
\item[\hskip \labelsep {\bfseries #1}]}{\end{trivlist}}
\newcommand{\argmax}{\operatornamewithlimits{argmax~}}
\newcommand{\argmin}{\operatornamewithlimits{argmin~}}

\newcommand{\mb}[1]{{\boldsymbol{#1}}}

\newcommand{\vu}{\mb{u}}

\usepackage{bbm}
\newcommand{\balpha}{\bm{\alpha}}
\newcommand{\bbeta}{\bm{\beta}}
\newcommand{\btau}{\bm{\tau}}
\newcommand{\brho}{\bm{\rho}}

\newcommand{\bDelta}{\bm{\Delta}}
\newcommand{\bs}{\bm{s}}

 {\everymath{\displaystyle\everymath{}}\array}%
 {\endarray}
\usepackage{pgfpages}

\usepackage{relsize}
\usepackage{lmodern}

\usepackage[nodisplayskipstretch]{setspace}
\AtBeginDocument{%
  \setlength\abovedisplayskip{4pt}
  \setlength\belowdisplayskip{4pt}
  }

\usepackage{listings}

\definecolor{codegreen}{rgb}{0,0.6,0}
\definecolor{codegray}{rgb}{0.5,0.5,0.5}
\definecolor{codepurple}{rgb}{0.58,0,0.82}
\definecolor{backcolour}{rgb}{0.95,0.95,0.92}

\lstdefinestyle{mystyle}{
    backgroundcolor=\color{backcolour},   
    commentstyle=\color{codegreen},
    keywordstyle=\color{magenta},
    numberstyle=\tiny\color{codegray},
    stringstyle=\color{codepurple},
    basicstyle=\ttfamily\footnotesize,
    breakatwhitespace=false,         
    breaklines=true,                 
    captionpos=b,                    
    keepspaces=true,                 
    numbers=left,                    
    numbersep=5pt,                  
    showspaces=false,                
    showstringspaces=false,
    showtabs=false,                  
    tabsize=2
}

\definecolor{darkbluetemp}{rgb}{0,0.08,0.45}
\usepackage[colorlinks=true, allcolors=darkbluetemp]{hyperref}

%% file: abstract-concise.tex
\begin{abstract}
Realistic use of neural networks often requires adhering to multiple constraints on latency, energy and memory among others.
A popular approach to find fitting networks is through constrained Neural Architecture Search (NAS), however, previous methods enforce the constraint only softly. 
Therefore,  the resulting networks do not exactly adhere to the resource constraint and their accuracy is harmed.
In this work we resolve this by introducing \textit{Hard Constrained diffeRentiable NAS (HardCoRe-NAS)}, that is based on an accurate formulation of the expected resource requirement and a scalable search method that satisfies the hard constraint throughout the search.
Our experiments show that HardCoRe-NAS generates state-of-the-art architectures, surpassing other NAS methods, while strictly satisfying the hard resource constraints without any tuning required.
\end{abstract}


%% file: acc_vs_lat_graph.tex
\begin{figure}[t]
\begin{adjustbox}{width=0.5\textwidth}
\begin{tikzpicture}
\begin{axis}[
            axis x line=left,
            axis y line=left,
            xmajorgrids=true,
            ymajorgrids=true,
            grid=both,
            xlabel style={below=1ex},
            enlarge x limits,
            ymin = 74.0,
            ymax = 78.1,
            xmin = 39.0,
            xmax = 61.0,
            xtick = {37,39,...,61},
            ytick = {74.0,74.5,...,78.1},
            ylabel = Top-1 Accuracy (\%),
            xlabel = Latency  (milliseconds),
            legend pos=north west,
            legend style={nodes={scale=0.8, transform shape}}
    ]


\addplot[color=red, mark=triangle*,very thick]coordinates {(37, 75.3)(40,75.8)(44, 76.4)(50, 77.1)(55, 77.6)(61,78.0)};\addlegendentry{HardCoRe-NAS}




\addplot[only marks,color=blue,mark=*,very thick]coordinates {(45, 75.2)} node (MobileNet_V3){};
\node[color=blue, above right = -0.1cm and -0.3cm of MobileNet_V3, font=\tiny] {MobileNet V3};

\addplot[only marks,color=blue,mark=*,very thick]coordinates {(40, 74.4)} node (tfnasa){};
\node[color=blue, right of=tfnasa, font=\tiny, node distance = 1.2cm] {TF-NAS-CPU A};

\addplot[only marks,color=blue,mark=*,very thick]coordinates {(60, 75.8)} node (tfnasb){};
\node[color=blue, left of=tfnasb, font=\tiny, node distance = 1.1cm] {TF-NAS-CPU B};


\addplot[only marks,color=blue,mark=*,very thick]coordinates {(55, 75.2)} node (mnasa1){};
\node[color=blue, above right = 0.0cm and -0.5cm of mnasa1, font=\tiny, node distance = 0.01cm] {MNAS-NET-A1};

\addplot[only marks,color=blue,mark=*,very thick]coordinates {(39, 74.5)} node (mnasb1){};
\node[color=blue, above right = 0.0cm and -0.5cm of mnasb1, font=\tiny, node distance = 0.01cm] {MNAS-NET-B1};

\addplot[only marks,color=blue,mark=*,very thick]coordinates {(41, 74.1)} node (spnas){};
\node[color=blue, right of=spnas, font=\tiny, node distance = 0.9cm] {SPNAS-NET};


\addplot[only marks,color=blue,mark=*,very thick]coordinates {(47, 75.1)} node (fbnet){};
\node[color=blue, right of=fbnet, font=\tiny, node distance = 0.7cm] {FB-NET};

\addplot[only marks,color=blue,mark=*,very thick]coordinates {(61, 76.7)} node (FairNASC){};
\node[color=blue, left of=FairNASC, font=\tiny, node distance = 0.8cm] {FairNAS-C};


\end{axis}
\end{tikzpicture}
\end{adjustbox}
\vspace*{-7mm}
\caption{Top-1 accuracy on Imagenet vs Latency measured on Intel Xeon CPU for a batch size of 1. HardCoreNAS can generate a network for any given latency, with accuracy according to the red curve, which is higher than all previous methods.}
\label{fig:acc_nas}
\end{figure}
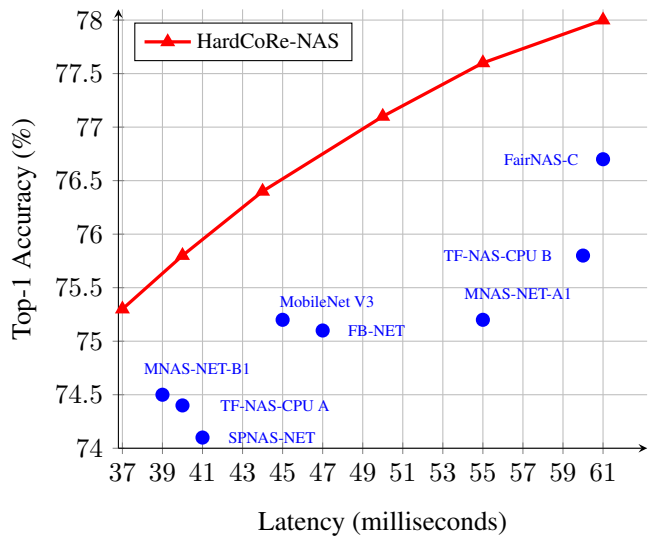

%% file: introduction.tex
\section{Introduction}
With the rise in popularity of Convolutional Neural Networks (CNN), the need for neural networks with fast inference speed and high accuracy, has been growing continuously. At first, manually designed architectures, such as VGG~\cite{VGG} or ResNet~\cite{ResNet}, targeted powerful GPUs as those were the common computing platform for deep CNNs. Many variants of those architectures were the golden standard until the need for deployment on edge devices and standard CPUs emerged. These are more limited computing platforms, requiring lighter architectures that for practical scenarios have to comply with hard constraints on the real time latency or power consumption. This has spawned a line of research aimed at finding architectures with both high performance and bounded resource demand. 

The main approaches to solve this evolved from Neural Architecture Search (NAS)~\cite{zoph2016neural, liu2018darts, cai2018proxylessnas}, while adding a constraint on the target latency over various platforms, e.g., TPU, CPU, Edge-GPU, FPGA, etc. 
%
%
The constrained-NAS methods can be grouped into two categories: (i) Reward based methods such as Reinforcement-Learning (RL) or Evolutionary Algorithm (EA) \cite{OFA,tan2019mnasnet, effnet, mobilenetv3}, where the search is performed by sampling networks and predicting their final accuracy and resource demands by evaluation over some validation set. The predictors are expensive to acquire and oftentimes inaccurate. (ii) Resource-aware gradient based methods \cite{TF-NAS,fbnet} formulate a differentiable loss function consisting of a trade-off between an accuracy term and a proxy soft penalty term. Therefore, the architecture can be directly optimized using stochastic gradient descent (SGD)~\cite{SGD}, however, it is hard to tune the trade-off between accuracy and resources, which deteriorates the network accuracy and fails to fully meet the resource requirements. 
The hard constraints over the resources are further violated due to a final discretization step projecting the architecture over the differentiable search space into the discrete space of architectures.
 
In this paper, we propose a search algorithm that produces architectures with high accuracy (Figure~\ref{fig:acc_nas}) that strictly satisfy any given hard latency constraint (Figure~\ref{fig:latency}). The search algorithm is fast and scalable to a large number of platforms. The proposed algorithm is based on several key ideas, starting from formulating the NAS problem more accurately, accounting for hard constraints over resources, and solving every aspect of it rigorously. For clarity we focus in this paper on latency constraints, however, our approach can be generalized to other types of resources.

 At the heart of our approach lies a suggested differentiable search space that induces a one-shot model~\cite{bender2018understanding, fairnas, SPOS, OFA} that is easy to train via a simple, yet effective, technique for sampling multiple sub-networks from the one-shot model, such that each one is properly pretrained. 
 We further suggest an accurate formula for the expected latency of every architecture residing in that space. Then, we search the space for sub-networks by solving a hard constrained optimization problem while keeping the one-shot model pretrained weights frozen. We show that the constrained optimization can be solved via the block coordinate stochastic Frank-Wolfe (BC-SFW) algorithm~\cite{SFW,BCFW}. Our algorithm converges faster than SGD, while tightly satisfying the hard latency constraint continuously throughout the search, including during the final discretization step.
 
 The approach we propose has several advantages. First, the outcome networks provide high accuracy and closely comply to the latency constraint. 
 In addition, our solution is scalable to multiple target devices and latency demands. This scalability is due to the efficient pretraining of the one-shot model as well as the fast search method that involves a relatively small number of parameters, governing only the structure of the architecture. 
 We hope that our formulation of NAS as a constrained optimization problem, equipped with an efficient algorithm that solves it, could give rise to followup work incorporating a variety of resource and structural constraints over the search space. 

%% file: related.tex
\section{Related Work}
\textbf{Efficient Neural Networks} are designed to meet the rising demand of deep learning models for numerous tasks per hardware constraints. Manually-crafted architectures such as MobileNets~\cite{howard2017mobilenets,sandler2018mobilenetv2} and ShuffleNet~\cite{zhang2018shufflenet} were designed for mobile devices, while TResNet~\cite{ridnik2020tresnet} and ResNesT~\cite{zhang2020resnest} are tailor-made for GPUs. Techniques for improving efficiency include pruning of redundant channels~\cite{dong2019network,aflalo2020knapsack} and layers~\cite{han2015learning}, model compression~\cite{han2015deep, he2018amc} and weight quantization methods~\cite{hubara2016binarized,umuroglu2017finn}. Dynamic neural networks adjust models based on their inputs to accelerate the inference, via gating modules~\cite{wang2018skipnet}, graph branching~\cite{huang2017multi} or dynamic channel selection~\cite{lin2017runtime}. These techniques are applied on  predefined architectures, hence cannot utilize or satisfy specific hardware constraints. 

\textbf{Neural Architecture Search} methods automate models' design per provided constraints. Early methods like NASNet~\cite{zoph2016neural} and AmoebaNet~\cite{real2019regularized} focused solely on accuracy, producing SotA classification models~\cite{huang2019gpipe} at the cost of GPU-years per search, with relatively large inference times. DARTS~\cite{liu2018darts} introduced a differential space for efficient search and reduced the training duration to days, followed by XNAS~\cite{nayman2019xnas} and ASAP~\cite{noy2020asap} that applied pruning-during-search techniques to further reduce it to hours. 
Hardware-aware methods such as ProxylessNAS~\cite{cai2018proxylessnas}, Mnasnet~\cite{tan2019mnasnet}, FBNet~\cite{fbnet} and TFNAS~\cite{TF-NAS} 
produce architectures that satisfy the required constraints by applying simple heuristics such as soft penalties on the loss function.
OFA~\cite{OFA} proposed a scalable approach across multiple devices by training an one-shot model~\cite{brock2017smash,bender2018understanding} for 1200 GPU hours. This provides a strong pretrained super-network being highly predictive for the accuracy of extracted sub-networks~\cite{SPOS,fairnas,yu2020train}. This work relies on such one-shot model acquired within only 400 GPU hours in a much simpler manner and satisfies hard constraints tightly with less heuristics.

%

\textbf{Frank-Wolfe (FW) algorithm}~\cite{frank_wolfe} 
is commonly used by machine learning applications~\cite{sun2019survey} thanks to its projection-free property~\cite{combettes2020projection,hazan2020faster} and ability to handle structured constraints~\cite{jaggi2013revisiting}. Modern adaptations aimed at deep neural networks (DNNs) optimization include more efficient variants~\cite{zhang2020quantized,combettes2020projection}, task-specific variants~\cite{chen2020frank,tsiligkaridis2020frank}, as well as improved convergence guarantees~\cite{lacoste2015global,d2020global}. Two prominent variants are the stochastic FW~\cite{hazan2016variance} and Block-coordinate FW~\cite{lacoste2013block}.
While FW excels as an optimizer for DNNs~\cite{berrada2018deep,pokutta2020deep}, this work is the first to utilize it for NAS.

%
%
%

%% file: method.tex
\vspace{-1.5mm}
\section{Method} \label{sec:method}
In the scope of this paper, we focus on latency-constrained NAS, searching for an architecture with the highest validation accuracy under a predefined latency constraint, denoted by $T$. Our architecture search space $\mathcal{S}$ is parametrized by $\zeta\in\mathcal{S}$, governing the architecture structure in a fully differentiable manner, and $w$, the convolution weights.
A latency-constrained NAS can be expressed as the following constrained bilevel optimization problem:
\begin{align}
&\min_{\zeta \in \mathcal{S}}  \mathbb{E}_{
\begin{array}{c}
 x, y \sim\mathcal{D}_{val} \\
\hat{\zeta} \sim \mathcal{P}_\zeta(\mathcal{S})
\end{array}
}[\mathcal{L}_{CE}(x, y \mid w^*, \hat{\zeta})] \notag\\
\label{eqn:NAS_bilevel}
&\text{s.t. } 
\text{LAT}(\zeta)\leq T \\
&w^*=\argmin_{w} \mathbb{E}_{
\begin{array}{c}
 x, y \sim\mathcal{D}_{train} \\
\hat{\zeta} \sim \mathcal{P}_\zeta(\mathcal{S})
\end{array}
}[\mathcal{L}_{CE}(x, y \mid w, \hat{\zeta}) ]\notag
\end{align}
where $\mathcal{D}_{train}$ and $\mathcal{D}_{val}$ are the train and validation sets' distributions respectively, $\mathcal{P}_\zeta(\mathcal{S})$ is some probability measure over the search space parameterized by $\zeta$, $\mathcal{L}_{CE}$ is the cross entropy loss as a differentiable proxy for the negative accuracy and $LAT(\zeta)$ is the estimated latency of the model.

To solve problem~\eqref{eqn:NAS_bilevel}, we construct a fully differentiable search space parameterized by $\zeta=(\balpha, \bbeta)$ (Section~\ref{sec:search_space}), that enables the formulation of a differentiable closed form formula expressing the estimated latency $LAT(\balpha, \bbeta)$ (Section~\ref{sec:latency_formula}) and efficient acquisition of $w^*$ (Section~\ref{sec:inner_problem}). Finally, we introduce rigorous constrained optimization techniques for solving the outer level problem (Section~\ref{sec:outer_problem}).

\subsection{Search Space}
\label{sec:search_space}
Aiming at the most accurate latency, a flexible search space is composed of a micro search space that controls the internal structures of each block $b\in\{1,..,d=4\}$, together with a macro search space that specifies the way those are connected to one another in every stage $s\in\{1,..,S=5\}$.

\subsubsection{Micro-Search}\label{sec:micro_search}
Every block is an \textit{elastic} version of the MBInvRes block, introduced in~\cite{mobilenetv2}, with expansion ratio $er\in\mathcal{A}_{er}=\{3,4,6\}$ of the point-wise convolution, kernel size $k\in\mathcal{A}_{k}=\{3\times3, 5\times5\}$ of the depth-wise separable convolution (DWS), and Squeeze-and-Excitation (SE) layer~\cite{SE} $se\in\mathcal{A}_{se}=\{\text{on}, \text{off}\}$.
The blocks are configurable, as illustrated at the bottom of Figure~\ref{fig:super_net}, using a parametrization $(\alpha^s_{b,er}, \alpha^s_{b,k}, \alpha^s_{b,se})$, defined for every block $b$ of stage $s$:
\begin{align*}
    \alpha^s_{b,er\setminus k \setminus se}\in \{0, 1\}  \quad \forall er\in\mathcal{A}_{er},\, k \in \mathcal{A}_{k},\, se\in \mathcal{A}_{se}
    \\
    \Sigma_{er \in \mathcal{A}_{er}} \alpha^s_{b,er} = 1
    \, ; \,
    \Sigma_{k \in \mathcal{A}_{k}} \alpha^s_{b,k} = 1
    \,; \,
    \Sigma_{se \in \mathcal{A}_{se}} \alpha^s_{b,k} = 1
\end{align*}
\begin{figure}[t]
  \centering
  \includegraphics[width=0.45\textwidth]{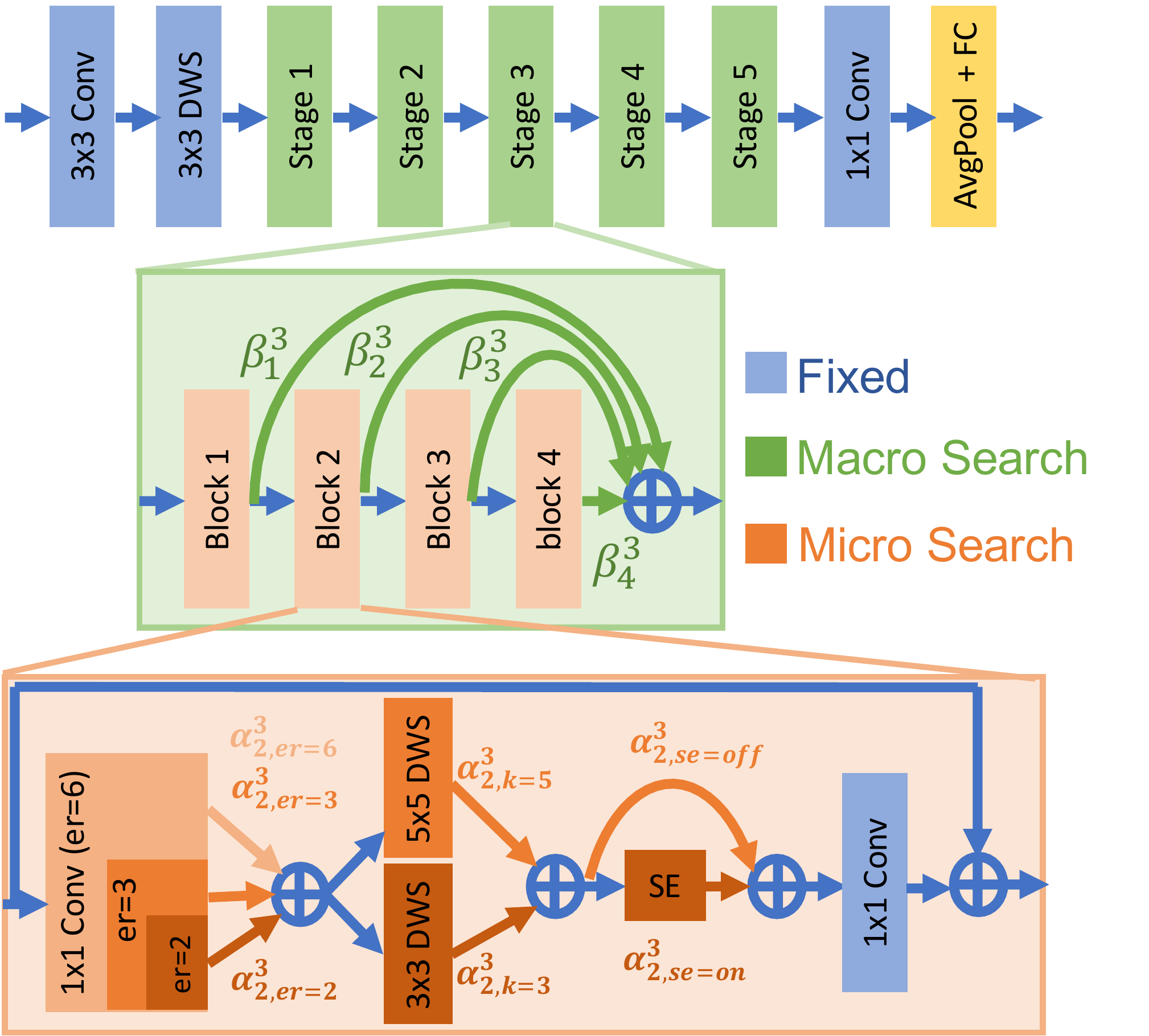}
\caption{A search space view via the one-shot model}
\label{fig:super_net}
\end{figure}
Each triplet $(er, k, se)$ induces a block configuration $c$ that resides within a micro-search space $\mathcal{C}=\mathcal{A}_{er}\otimes\mathcal{A}_{k}\otimes\mathcal{A}_{se}$, parameterized by  $\balpha\!\in\!\mathcal{A}=\bigotimes_{s=1}^S\bigotimes_{b=1}^d\bigotimes_{c\in\mathcal{C}}\alpha^s_{b,c}$, where $\otimes$ denotes the Cartesian product. Hence, for each block $b$ of stage $s$ we have:
\begin{align*}
    \alpha^s_{b,c}\in \{0, 1\}^{\left|\mathcal{C}\right|} \quad &; \quad 
    \Sigma_{c\in\mathcal{C}} \alpha^s_{b,c} = 1
\end{align*}
An input feature map $x^s_b$ to block $b$ of stage $s$ is processed as follows:
\begin{align*}
    x^s_{b+1}=\sum_{c\in\mathcal{C}}\alpha^s_{b,c}\cdot O^s_{b,c}(x^s_b) 
\end{align*}
where $O^s_{b,c}(\cdot)$ is the operation performed by the elastic MBInvRes block configured according to $c=(er,k,se)$.

Having several configurations $O^s_{b,c}(\cdot)$ share the same values of $\alpha^s_{b,er}$ or $\alpha^s_{b,k}$ or $\alpha^s_{b,se}$ induces weight sharing between the common operations of the associated architectures. This weight sharing is beneficial for solving the inner problem~\eqref{eqn:NAS_bilevel} effectively, as will be discussed in Section~\ref{sec:inner_problem}.

\subsubsection{Macro-Search}\label{sec:macro_search}
Inspired by~\cite{TF-NAS, OFA}, the output of each block of every stage is also directed to the end of the stage as illustrated in the middle of Figure~\ref{fig:super_net}. Thus, the depth of each stage $s$ is controlled by the parameters $\bbeta\in\mathcal{B}=\bigotimes_{s=1}^S\bigotimes_{b=1}^d \beta^s_b$, such that:
\begin{align*}
    \beta^s_b\in \{0, 1\}^d \quad &; \quad 
    \Sigma_{b=1}^d \beta^s_b = 1
\end{align*}
The depth is $b^s\!\in\!\left\{b\mid\beta^s_b=1, b\in\{1,..,d\}\right\}$, since:
\begin{align*}
    x^{s+1}_1=\Sigma_{b=1}^{d}\beta^s_b\cdot x^s_{b+1}
\end{align*}

\subsubsection{The Composed Search Space}\label{sec:composed_search_space}
The overall search space is composed of both the micro and macro search spaces parameterized by $\balpha\in\mathcal{A}$ and $\bbeta\in\mathcal{B}$, respectively, such that:
\begin{align}
   \mathcal{S} = \left\{(\balpha, \bbeta)\left\vert
   \begin{array}{c}
    \balpha\in\mathcal{A}, \bbeta\in\mathcal{B}
    \\
     \alpha^s_{b,c}\in \{0, 1\}^{\left|\mathcal{C}\right|} \,; \, 
    \Sigma_{c\in\mathcal{C}} \alpha^s_{b,c} = 1  \\
   \beta^s_b\in \{0, 1\}^d \,; \,  
    \Sigma_{b=1}^d \beta^s_b = 1 
    \\
    \forall s\in\{1,..,S\}, b\in\{1,..,d\},c\in\mathcal{C}
   \end{array}\right.
   \right\}
   \notag
\end{align}

A continuous probability distribution is induced over the space, by relaxing $\alpha^s_{b,c}\in \{0, 1\}^{\left|\mathcal{C}\right|} \rightarrow\alpha^s_{b,c}\in \mathbb{R}_+^{\left|\mathcal{C}\right|}$ and $ \beta^s_b\in \{0, 1\}^d \rightarrow\beta^s_b\in \mathbb{R}_+^d$ to be continuous rather than discrete.
A sample sub-network is drawn $\hat{\zeta}=(\hat{\balpha}, \hat{\bbeta})\sim \mathcal{P}_\zeta(\mathcal{S})$ using the Gumbel-Softmax Trick~\cite{GumbelSM} such that $\hat{\zeta}\in\mathcal{S}$, as specified in~\eqref{eqn:gumbel_alpha} and~\eqref{eqn:gumbel_beta}.
In summary, one can view the parametrization $(\balpha, \bbeta)$ as a composition of probabilities in $\mathcal{P}_{\zeta}(\mathcal{S})$ or as degenerated one-hot vectors in $\mathcal{S}$.

Effectively we include at least a couple of blocks in each stage by setting $\beta^s_1\equiv 0$, hence, the overall size of the search space is:
\begin{align*}
&|\mathcal{S}|=(\sum_{b=2}^d \left|\mathcal{C}\right|^b)^S =
(\sum_{b=2}^d \left|\mathcal{A}_{er}\right|^b\cdot\left|\mathcal{A}_{k}\right|^b\cdot\left|\mathcal{A}_{se}\right|^b)^S \\
&=((3\times 2\times 2)^2+(3\times 2 \times 2)^3+(3\times 2\times 2)^4)^5\approx 10^{27}
\end{align*}

\subsection{Formulating the Latency Constraint}
\label{sec:latency_formula}
Aiming at tightly satisfying latency constraints, we propose an accurate formula for the expected latency of a sub-network. The expected latency of a block $b$ can be computed by summing over the latency $t^s_{b,c}$ of every possible configuration $c\in\mathcal{A}$:
$$\bar{\ell}^s_{b}=\Sigma_{c\in\mathcal{C}} \alpha^s_{b,c} \cdot t^s_{b,c}$$
Thus the expected latency of a stage of depth $b'$ is 
\begin{align}\label{eqn:latency_correction}
\ell^s_{b'}=\sum_{b=1}^{b'}\bar{\ell}^s_{b}
\end{align}
Taking the expectation over all possible depths yields  
$$\ell^s= \sum_{b'=1}^d \beta^s_{b'}\cdot \ell^s_{b'}$$
and summing over all the stages results in the following formula for the overall latency:
\begin{align}\label{eqn:latency_formula}
   LAT(\balpha, \bbeta) = 
   \sum_{s=1}^S \sum_{b'=1}^d \sum_{b=1}^{b'}\sum_{c\in\mathcal{C}} \alpha^s_{b,c} \cdot t^s_{b,c} \cdot \beta^s_{b'}
\end{align}
The the summation originated in \eqref{eqn:latency_correction} differentiates our latency formulation~\eqref{eqn:latency_formula} from that of~\cite{TF-NAS}. 
%

Figure~\ref{fig:latency} provides empirical validation of~\eqref{eqn:latency_formula}, showing that in practice the actual and estimated latency are very close on both GPU and CPU. More details on the experiments are provided in Section~\ref{sec:validate_latency}.
\input{latency_graph}

\textbf{Remark:} By replacing the latency measurements $t^s_{b,c}$ with other resources, e.g., memory, FLOPS, MACS, etc., one can use formula~\eqref{eqn:latency_formula} to add multiple hard constraints to the outer problem of~\eqref{eqn:NAS_bilevel}.

\subsection{Solution to the Inner Problem $w^*$}
\label{sec:inner_problem}
Previous work proposed approximated solutions to the following unconstrained problem:
\begin{align*}
&\min_{\zeta \in \mathcal{S}} \mathbb{E}\left[\mathcal{L}_{CE}\left(x, y \mid w^*_\zeta, \zeta\right)\right] \notag\\
&\text{s.t. }
w^*_\zeta=\argmin_{w} \mathbb{E}_{\zeta\sim\mathcal{P}_\zeta(\mathcal{S})}\left[\mathcal{L}_{CE}\left(x, y \mid w, \zeta\right)\right] \notag
\end{align*}
typically by alternating or simultaneous updates of $w$ and $\zeta$ \cite{liu2018darts, SNAS, cai2018proxylessnas, fbnet, TF-NAS}.
This approach has several limitations.
First, obtaining a designated $w^*_\zeta$ with respect to every update of $\zeta$ involves a heavy training of a neural network until convergence. Instead a very rough approximation is obtained by just a few update steps for $w$.
In turn, this approximation creates a bias towards strengthening networks with few parameters since those learn faster, hence, get sampled even more frequently, further increasing the chance to learn in a positive feedback manner. Eventually, often overly simple architectures are generated, e.g., consisting of many skip-connections~\cite{P-DARTS, DARTS+}. Several remedies have been proposed, e.g., temperature annealing, adding uniform sampling, modified gradients and updating only $w$ for a while before the joint optimization begins~\cite{noy2020asap, fbnet, TF-NAS, nayman2019xnas}. While those mitigate the bias problem, they do not solve it.

We avoid such approximation whatsoever. Instead we obtain $w^*$ of the inner problem of~\eqref{eqn:NAS_bilevel} only once, with respect to a uniformly distributed architecture, sampling $\bar{\zeta}$ from $\mathcal{P}_{\bar{\zeta}}(\mathcal{S})=U(\mathcal{S})$.

This is done by sampling multiple distinctive paths (sub-networks of the one-shot model) for every image in the batch in an efficient way (just a few lines of code provided in the supplementary materials), using the Gumbel-Softmax Trick. 
For every feature map $x$ that goes through block $b$ of stage $s$, distinctive uniform random variables $\mathcal{U}^s_{b,c,x}, \mathcal{U}^s_{b,x}\!\sim\!U(0,1)$ are sampled, governing the path undertaken by this feature map:
\begin{align}
    \label{eqn:gumbel_alpha}
    &\hat{\alpha}^s_{b,c,x} = \frac{e^{\log(\alpha^s_{b,c}) -\log(\log(\mathcal{U}^s_{b,c,x}))}}{\sum_{c\in \mathcal{C}} e^{\log(\alpha^s_{b,c}) -\log(\log(\mathcal{U}^s_{b,c,x}))}}
    \\
    \label{eqn:gumbel_beta}
    &\hat{\beta}^s_{b,x} = \frac{e^{\log(\beta^s_b) -\log(\log(\mathcal{U}^s_{b,x}))}}{\sum_{c\in \mathcal{C}} e^{\log(\alpha^s_{b,c}) -\log(\log(\mathcal{U}^s_{b,x}))}}
\end{align}


Based on the observation that the accuracy of a sub-network with $w^*$ should be predictive for its accuracy when optimized as a stand-alone model from scratch, we aim at an accurate prediction.
Our simple approach implies that, with high probability, the number of paths sampled at each update step is as high as the number of images in the batch. This is two orders of magnitude larger than previous methods that sample a single path per update step \cite{SPOS, OFA}, while avoiding the need to keep track of all the sampled paths \cite{fairnas}. Using multiple paths reduces the variance of the gradients with respect to the paths sampled by an order of magnitude\footnote{A typical batch consists of hundreds of i.i.d paths, thus a variance reduction of the square root of that is in place.}. Furthermore, leveraging the weight sharing implied by the structure of the elastic MBInvRes block (Section~\ref{sec:micro_search}), the number of gradients observed by each operation is increased by a factor of at least $\frac{|\mathcal{C}|}{\max(|\mathcal{A}_{er}|,|\mathcal{A}_{k}|,|\mathcal{A}_{se}|)}=\frac{3\times2\times2}{\max(3,2,2)}\approx 4$. This further reduces the variance by half. 

\input{kendall_tau}
Figure~\ref{fig:acc_kendal_tau} shows that we obtain a one-shot model $w^*$ with high correlation between the ranking of sub-networks directly extracted from it and the corresponding stand-alone counterpart trained from scratch. See more details in Section~\ref{sec:evaluate_inner}. This implies that $w^*$ captures well the quality of sub-structures in the search space.
\subsection{Solving the Outer Problem}
\label{sec:outer_problem}
Having defined the formula for the latency LAT$(\zeta)$ and obtained a solution for $w^*$, we can now continue to solve the outer problem~\eqref{eqn:NAS_bilevel}. 

\subsubsection{Searching Under Latency Constraints}
Most differentiable resource-aware NAS methods account for the resources through shaping the loss function with soft penalties~\cite{fbnet, TF-NAS}. This approach solely does not meet the constraints tightly.
Experiments illustrating this are described in Section~\ref{subsec:toy_example}.

Our approach directly solves the constrained outer problem~\eqref{eqn:NAS_bilevel}, hence, it enables the strict satisfaction of resource constraints by further restricting the search space, i.e.,
$\mathcal{S}_{LAT}=\{\zeta \mid \zeta\in \mathcal{P}_\zeta(\mathcal{S}), \text{LAT}(\zeta)\leq T\}$.

As commonly done for gradient based approaches, e.g.,~\cite{liu2018darts}, we relax the discrete search space $\mathcal{S}$ to be continuous by searching for $\zeta\in \mathcal{S}_{LAT}$.
As long as $\mathcal{S}_{LAT}$ is convex, it could be leveraged for applying the stochastic Frank-Wolfe (SFW) algorithm~\cite{SFW} to directly solve the constrained outer problem:
\begin{align}\label{eqn:constrained_outer_problem}
&\min_{\zeta \in \mathcal{S}_{LAT}} \mathbb{E}_{x, y \sim\mathcal{D}_{val}}\left[\mathcal{L}_{CE}\left(x, y \mid w^*, \zeta\right)\right] 
\end{align}
following the update step:
\begin{align}\label{eqn:sfw_step}
\zeta_{t+1} &= (1-\gamma_t)\cdot\zeta_t + \gamma_t\cdot\xi_t\\
\xi_t&=\argmin_{\zeta \in \mathcal{S}_{LAT}}\zeta^T\cdot \nabla_{\zeta_t}\mathcal{L}_{CE}\left(x_t, y_t \mid w^*, \zeta_t\right)
\label{eqn:lp}
\end{align}
where $(x_t, y_t)$ and $\gamma_t$ are the sampled data and the learning rate at step $t$, respectively.
For $\mathcal{S}_{LAT}$ of linear constraints, the linear program~\eqref{eqn:lp} can be solved efficiently, using the Simplex algorithm~\cite{simplex}.

A convex $\mathcal{S}_{LAT}$ together with $\gamma_t\in[0,1]$ satisfy $\zeta_t\in\mathcal{S}_{LAT}$ anytime, as long as $\zeta_0\in\mathcal{S}_{LAT}$.
 We provide a method for satisfying the latter in the supplementary materials.

The benefits of such optimization are demonstrated in Figure~\ref{fig:fw_vs_gd} through a toy problem, described in Section~\ref{subsec:toy_example}. While GD is sensitive to the trade-off involved with a soft penalty, FW converges faster to the optimum with zero penalty. 

\begin{figure}[htb]
    \centering
    \includegraphics[width=0.47\textwidth]{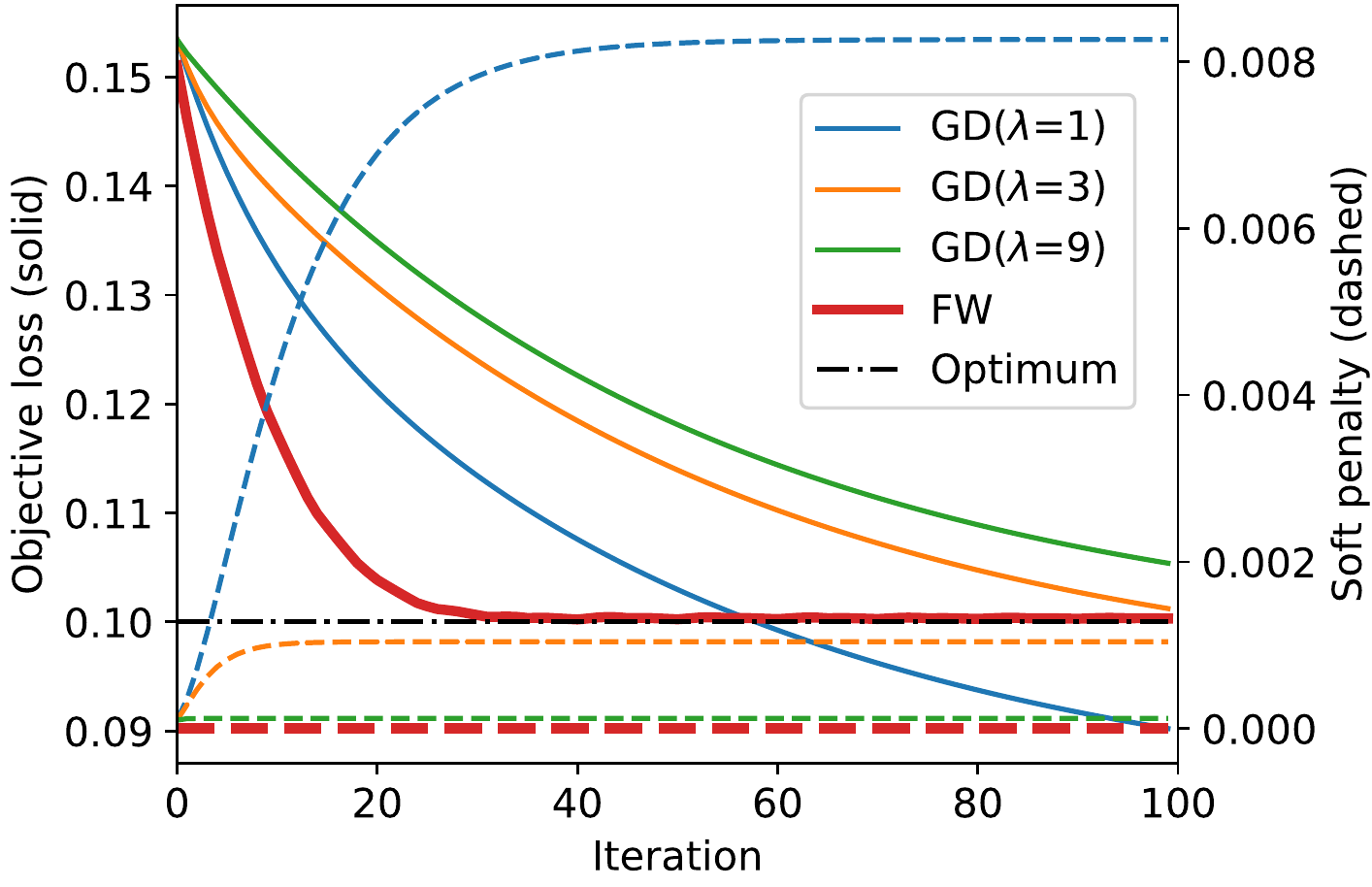}
    \caption{Objective loss and soft penalty for FW and GD for different values of the penalty coefficient $\lambda$. Solid and dashed curves represent the objective loss (left y-axis) and soft penalty (right y-axis), respectively. FW converges to the optimum much faster than GD.}
    \label{fig:fw_vs_gd}
\end{figure}

All this requires $\mathcal{S}_{LAT}$ to be convex. 
While $\mathcal{P}_\zeta(\mathcal{S})$ is obviously convex, formed by linear constraints, the latency constraint $\text{LAT}(\zeta)\leq T$ is not necessarily so. The latency formula \eqref{eqn:latency_formula} can be expressed as a quadratic constraint by constructing a matrix $\Theta\in\mathbb{R}_+^{S\cdot d\cdot|\mathcal{A}|\times S\cdot d}$ from $t^s_{b,c}$, such that,
\begin{align}\label{eqn:latency_matrix}
   LAT(\zeta)=LAT(\balpha, \bbeta) = \balpha^T \Theta \bbeta \quad ; \quad \zeta\in\mathcal{P}_{\zeta}(\mathcal{S})
\end{align}
Since $\Theta$ is constructed from measured latency, it is not guaranteed to be positive semi-definite, hence, the induced quadratic constraint could make $\mathcal{S}_{LAT}$ non-convex.

To overcome this, we introduce the Block Coordinate Stochastic Frank-Wolfe (BCSFW) Algorithm~\ref{alg:BCSFW}, that combines Stochastic Frank-Wolfe with Block Coordinate Frank-Wolfe~\cite{BCFW}. This is done by forming separated convex feasible sets at each step, induced by linear constraints only:
\begin{align}
    \label{eqn:alpha_search_space}
 \balpha_t \in \mathcal{S}^{\balpha}_t&=\{\balpha\mid \balpha\in \mathcal{A},& \bbeta_t^T \Theta^T \cdot \balpha&\leq T\}
 \\
    \label{eqn:beta_search_space}
 \bbeta_t \in \mathcal{S}^{\bbeta}_t&=\{\bbeta\mid \bbeta\in \mathcal{B},& \balpha_t^T \Theta \cdot \bbeta&\leq T\}
\end{align}
 This implies that $\zeta_t=(\balpha_t, \bbeta_t)\in\mathcal{S}_{LAT}$ for all $t$.
Moving inside the feasible domain at anytime avoids irrelevant infeasible structures from being promoted and hiding feasible structures.
\vspace{1em}
\input{fw_algo}

\subsubsection{Projection Back to the Discrete Space}\label{sec:projection}
As differentiable NAS methods are inherently associated with a continuous search space, a final discretizaiton step $\mathbb{P}:\mathcal{P}_\zeta(\mathcal{S})\rightarrow\mathcal{S}$ is required for extracting a single architecture. 
Most methods use the \textit{argmax} operator:
\begin{align}\label{eqn:argmax}
\bar{\alpha}^s_{b,c}&:=\mathbbm{1}\{c=\argmax_{c\in\mathcal{C}}\alpha^s_{b,c}\} 
\\\notag
\bar{\beta}^s_b&:=\mathbbm{1}\{b=\argmax_{b=2,..,d}\beta^s_b\} 
\end{align}
for all $s\in\{1,..,S\}, b\in\{1,..,d\},c\in\mathcal{C}$, where $(\bar{\balpha}, \bar{\bbeta})$ is the solution to the outer problem of~\eqref{eqn:NAS_bilevel}.

For resource-aware NAS methods, applying such projection results in possible violation of the resource constraints, due to the shift from the converged solution in the continuous space. Experiments showing that latency constraints are violated due to~\eqref{eqn:argmax} are provided in Section~\ref{sec:projection_effect}.

While several methods mitigate this violation by promoting sparse probabilities during the search, e.g.,~\cite{noy2020asap, nayman2019xnas}, 
our approach completely eliminates it by introducing an alternative projection step, described next.

Viewing the solution of the outer problem $(\balpha^*, \bbeta^*)$ as the credit assigned to each configuration,
we introduce a projection step that maximizes the overall credit while strictly satisfying the latency constraints. It is based on solving the following two linear programs:
\begin{align}\label{eqn:projection_step}
  \max_{\balpha\in\mathcal{S}^{\balpha^*}} \balpha^T\cdot \balpha^*
  \quad; \quad
  \max_{\bbeta\in\mathcal{S}^{\bbeta^*}} \bbeta^T\cdot\bbeta^*
\end{align}
Note, that when there is no latency constraint, e.g., $T\rightarrow\infty$, \eqref{eqn:projection_step} coincides with~\eqref{eqn:argmax}.

We next provide a theorem that guarantees that the projection~\eqref{eqn:projection_step} yields a sparse solution, representing a valid sub-network of the one-shot model. Specifically, a single probability vector from those composing $\balpha$ and $\bbeta$ contains up to two non-zero entries each, as all the rest are one-hot vectors. 
\begin{theo}
\label{thm:one_hot_sol}
The solution $(\balpha, \bbeta)$ of \eqref{eqn:projection_step} admits:
\begin{align*}
  &\sum_{c\in\mathcal{C}}|\alpha^s_{b,c}|^0 = 1 \,\forall 
  (s, b)\in\{1,..,S\}\otimes \{1,..,d\}\setminus{\{(s_{\balpha}, b_{\balpha})\}}
  \\
  &\sum_{b=1}^d|\beta^s_b|^0 = 1 \quad \forall s\in\{1,..,S\}\setminus{\{s_{\bbeta}\}}
\end{align*}
where $|\cdot|^0=\mathbbm{1}\{\cdot>0\}$ and $(s_{\balpha}$, $b_{\balpha})$, $s_{\bbeta}$  are single block and stage respectively, satisfying:
\begin{align}\label{eqn:dense_remiders}
  \sum_{c\in\mathcal{C}}|\alpha^{s_{\balpha}}_{{b_{\balpha}},c}|^0 \leq 2
  \quad ; \quad 
  \sum_{b=1}^d|\beta^{s_{\bbeta}}_b|^0 \leq 2
\end{align}
\end{theo}
Refer to the supplementary materials for the proof.

\textbf{Remark:} A negligible deviation is associated with taking the argmax~\eqref{eqn:argmax} over the only two couples referred to in~\eqref{eqn:dense_remiders}. Experiments supporting this are described in Section~\ref{sec:projection_effect}. Furthermore, this can be entirely eliminated by solving an equivalent Multiple-Choice Knapsack Problem (MCKP) as described in the supplementary material.

%% file: latency_graph.tex
\begin{figure}[t]
\centering
\begin{adjustbox}{width=0.47\textwidth}
\begin{tikzpicture}
\begin{axis}[
            axis x line=left,
            axis y line=left,
            xmajorgrids=true,
            ymajorgrids=true,
            grid=both,
            xlabel style={below=1ex},
            enlarge x limits,
            ymin = 24.0,
            ymax = 66,
            xmin = 24.0,
            xmax = 66,
            xtick = {25,30,...,65},
            ytick = {25,30,...,65},
            ylabel = Measured Latency (milliseconds),
            xlabel = Latency Formula (milliseconds),
            legend pos=north west,
            legend style={nodes={scale=0.8, transform shape}}
    ]

\addplot[only marks, color=red, mark=triangle*,very thick]coordinates {(25, 27)(30, 32)(35, 36)(40, 42)}	
;\addlegendentry{GPU}

\addplot[only marks, color=blue, mark=triangle*,very thick]coordinates {
(35, 35.36)
(40, 42.63)
(45, 47.45)
(50, 54.51)
(55, 56.98)
(60, 61.93)
(65, 64.03)
};\addlegendentry{CPU}

\addplot[style=dashed, color=ForestGreen, very thick]coordinates {(24,24)(66, 66)};
\addplot[color=ForestGreen,very thick]coordinates {(55,35)} node (kendal_tau_ours_100){
$\begin{array}{ll}
     \mathbf{y}&\mathbf{=} \mathbf{1.003\cdot x} \\
     \mathbf{R^2}&\mathbf{=}\mathbf{0.99} 
\end{array}$
};
\end{axis}
\end{tikzpicture}
\end{adjustbox}
\vspace*{-3mm}
\caption{Experimental results showing that the latency calculated with formula~\eqref{eqn:latency_formula} is very close to the latency measured in practice, on both CPU and GPU.}
\label{fig:latency}
\end{figure}

%% file: kendall_tau.tex
\begin{figure}[t]
\begin{adjustbox}{width=0.47\textwidth}
\begin{tikzpicture}
\begin{axis}[
            axis x line=left,
            axis y line=left,
            xmajorgrids=true,
            ymajorgrids=true,
            grid=both,
            xlabel style={below=1ex},
            enlarge x limits,
            ymin = 68.0,
            ymax = 77.7,
            xmin = 70.0,
            xmax = 77.7,
            xtick = {68,67,...,78},
            ytick = {68,67,...,78},
            ylabel = One-shot Accuracy (\%),
            xlabel = Stand-alone Accuracy (\%),
            legend pos=north west,
            legend style={nodes={scale=0.8, transform shape}}
    ]

\addplot[only marks, color=ForestGreen, mark=triangle*,very thick]coordinates {(70.95,69.68)(71.47, 70.5)(72, 71.8)(72.83, 72.60)(73.88, 73.34)(74.29, 73.89)
(73.24,73.25)(73.61,73.69)(74.32, 74.83)(74.82, 75.28)(75.80, 75.82)(76.12,76.03)
(74.62, 74.2)(74.87,74.7 )(75.62,75.66)(75.85,75.88)(76.72,76.43)(77.07,76.57)};\addlegendentry{Multi-path} 



\addplot[only marks, color=orange, mark=triangle*,very thick]coordinates {(70.95,61.036)(71.47,65.73)(72,68.374)(72.83,71.178)(73.88,71.518)(74.29,73.03)
(73.24,71.554)(73.61,72.798)(74.32,74.094)(74.82,74.194)(75.80,74.798)(76.12,74.274)
(74.62,73.278)(74.87,73.92)(75.62,74.636)(75.85,74.074)(76.72,74.322)(77.07,72.866)};\addlegendentry{Single-path} 



\addplot[style=dashed, color=black, very thick]coordinates {(66, 66)(78, 78)};

\addplot[color=ForestGreen,very thick]coordinates {(71,74)} node (kendal_tau_ours_100){\makecell{$\mathbf{\btau=0.95}$\\ $\mathbf{\brho=0.99}$}};



\addplot[color=orange,very thick]coordinates {(76, 70)} node (kendal_tau_other3){\makecell{$\mathbf{\btau=0.70}$\\ $\mathbf{\brho=0.82}$}};


\end{axis}
\end{tikzpicture}
\end{adjustbox}
\vspace*{-3mm}
\caption{Top-1 accuracy on Imagenet of networks trained from scratch v.s. corresponding sub-networks extracted from our one-shot model. $\tau$ and $\rho$ represent the Kendall-Tau and Spearman correlation coefficients, respectively.}
\label{fig:acc_kendal_tau}
\end{figure}

%% file: fw_algo.tex
\begin{algorithm}[htb]
   \caption{Block Coordinate SFW (BCSFW)}
   \label{alg:BCSFW}
\begin{algorithmic}
\INPUT $(\balpha_0, \bbeta_0) \in \mathcal{S}_{LAT}$
\FOR{$t=0,\dots,K$}
\STATE Pick $\delta:=\balpha$ or $\delta:=\bbeta$ at random
\STATE Sample an i.i.d validation batch $(x_t, y_t)\sim\mathcal{D}_{val}$
\STATE $\xi_t=\argmin_{\xi \in \mathcal{S}^{\delta}_t}\xi^T\cdot \nabla_{\delta_t}\mathcal{L}_{CE}\left(x_t, y_t \mid w^*, \delta_t\right)$
\STATE Update $\delta_{t+1} = (1-\gamma_t)\cdot\delta_t + \gamma_t\cdot\xi_t$ with $\gamma_t=\frac{4}{t+4}$
\ENDFOR
\end{algorithmic}
\end{algorithm}

%% file: exp.tex
\section{Experimental Results}\label{sec:experiments}

\subsection{Search for State-of-the-Art Architectures}
\subsubsection{Dataset and Setting} \label{sec:exp_setting}
For all of our experiments, we train our networks using SGD with a learning rate of $0.1$ with cosine annealing, Nesterov momentum of $0.9$, weight decay of $10^{-4}$, applying label smoothing \cite{label_smoothing} of 0.1, mixup \cite{mixup} of 0.2, Autoaugment \cite{autoaugment}, mixed precision and EMA-smoothing.

We obtain the solution of the inner problem $w^*$ as specified in sections \ref{sec:inner_problem} and \ref{sec:evaluate_inner} over 80\% of a random 80-20 split of the ImageNet train set. We utilize the remaining 20\% as a validation set and search for architectures with latencies of $40, 45, 50, 55, 60$ and $25, 30, 40$ milliseconds running with a batch size of 1 and 64 on an Intel Xeon CPU and and NVIDIA P100 GPU, respectively. The search is performed according to section \ref{sec:outer_problem} for only 2 epochs of the validation set, lasting for 8 GPU hours\cref{fn:batch_16}. 
\vspace{-0.2em}
\subsubsection{Comparisons with other methods}\label{sec:exp_comparison}
\vspace{-0.2em}
We compare our generated architectures to other state-of-the-art NAS methods in Table~\ref{tab:exp} and Figure~\ref{fig:acc_nas}. 
For each model in the table, we use the official PyTorch implementation \cite{pytorch} and measure its latency running on a single thread with the exact same conditions as for our networks. We excluded further optimizations, such as Intel MKL-DNN~\cite{mkl_dnn}, therefore, the latency we report may differ from the one originally reported. 
For the purpose of comparing the generated architectures alone, excluding the contribution of evolved pretraining techniques, all the models (but OFA\cref{fn:ofa}) are trained from a random initialization with the same settings, specified in section~\ref{sec:exp_setting}. 
It can be seen that networks generated by our method meet the latency target closely, while at the same time surpassing all the others methods on the top-1 accuracy over ImageNet with a reduced scalable search time. 


 

\input{big_table}
\vspace{-0.2em}
\subsection{Empirical Analysis of Key Components}
\vspace{-0.2em}
\subsubsection{Validation of the Latency Formula}\label{sec:validate_latency}
\vspace{-0.2em}
One of our goals is to provide a practical method to accurately meet the given resource requirements.
Hence, we validate empirically the accuracy of the latency formula~\eqref{eqn:latency_formula}, by comparing its estimation with the measured latency. Experiments were performed on two platforms: Intel Xeon CPU and NVIDIA P100 GPU, and applied to multiple networks. Results are shown in Figure~\ref{fig:latency}, which confirms a linear relation between estimated and measured latency, with a ratio of $1.003$ and a coefficient of determination of $R^2=0.99$. This supports the accuracy of the proposed formula.

\subsubsection{Evaluating the Solution of the Inner Problem $w^*$}\label{sec:evaluate_inner}
The ultimate quality measure for a generated architecture is arguably its accuracy over a test set when trained as a stand-alone model from randomly initialized weights.
To evaluate the quality of our one-shot model $ w^*$ we compare the accuracy of networks extracted from it with the accuracy of the corresponding architectures when trained from scratch. 
Naturally, when training from scratch the accuracy could increase. However, a desired behavior is that the ranking of the accuracy of the networks will remain the same with and without training from scratch.
The correlation can be calculated via the Kendall-Tau~\cite{KendallTau} and Spearman's~\cite{spearman1961general} rank correlation coefficients, denoted as $\tau$ and $\rho$, respectively.

To this end, we first train for 250 epochs a one-shot model $\bar{w}^*$ using the heaviest possible configuration, i.e., a depth of $4$ for all stages, with $er=6, k=5\times 5, se=on$ for all the blocks. Next, to obtain $w^*$, we apply the multi-path training of Section~\ref{sec:inner_problem} for additional 100 epochs of fine-tuning $\bar{w}^*$ over 80\% of a 80-20 random split of the ImageNet train set~\cite{imagenet_cvpr09}. The training settings are specified in Section~\ref{sec:exp_setting}. 
The first 250 epochs took 280 GPU hours\footnote{Running with a batch size of 200 on 8$\times$NVIDIA V100} and the additional 100 fine-tuning epochs took 120 GPU hours\footnote{\label{fn:batch_16}Running with a batch size of 16 on 8$\times$NVIDIA V100}, summing to a total of 400 hours on NVIDIA V100 GPU to obtain $w^*$.
To further demonstrate the effectiveness of our multi-path technique, we repeat this procedure also without it, sampling a single path for each batch.

For the evaluation of the ranking correlations, we extract 18 sub-networks of common configurations for all stages of depths in $\{2,3,4\}$ and blocks with an expansion ratio in $\mathcal{A}_{er}=\{3,4,6\}$, a kernel size in $\mathcal{A}_k=\{3\times 3,5\times 5\}$ and Squeeze and Excitation being applied. 
We train each of those as stand-alone from random initialized weights for 200 epochs over the full ImageNet train set, and extract their final top-1 accuracy over the validation set of ImageNet.

Figure~\ref{fig:acc_kendal_tau} shows for each extracted sub-network its accuracy without and with stand-alone training. It further shows results for both multi-path and single-path sampling.
 It can be seen that the multi-path technique improves $\tau$ and $\rho$ by $0.35$ and $0.17$ respectively, leading to a highly correlated rankings of $\tau\!=\!0.95$ and $\rho\!=\!0.99$.




\vspace{1.5mm}
\textbf{2 for 1 - $w^*$ Bootstrap:}\\
A nice benefit of the training scheme described in this section is that it further shortens the generation of trained models. We explain this next.

The common approach of most NAS methods is to re-train the extracted sub-networks from scratch.
Instead, we leverage having two sets of weights: $\bar{w}^*$ and $w^*$. Instead of retraining the generated sub-networks from a random initialization we opt for fine-tuning $w^*$ guided by knowledge distillation~\cite{KD} from the heaviest model $\bar{w}^*$.
Empirically, we observe that this surpasses the accuracy obtained when training from scratch at a fraction of the time. 
More specifically, we are able to generate a trained model within a small marginal cost of 15 GPU hours. The total cost for generating $N$ trained models is $400 + 15N$, much lower than the $1200 + 25N$ reported by OFA~\cite{OFA}. See Table~\ref{tab:exp}.
This makes our method scalable for many devices and latency requirements. 
Note, that allowing for longer training further improves the accuracy significantly (see the supplementary materials).

\subsubsection{Outer Problem: Hard vs Soft }
\label{subsec:toy_example}
Next, we evaluate our method's ability to satisfy accurately a given latency constraint.
We compare our hard-constrained formulation~\eqref{eqn:NAS_bilevel} with the common approach of adding soft penalties to the loss function~\cite{TF-NAS, fbnet}.
The experiments were performed over a simple and intuitive toy problem:
\begin{equation}
    \label{eqn:toy_ex}
    \min_x \|x\|^2  \quad \text{s.t. }\ 
       \Sigma_{i=1}^d x_i = 1 \quad ; \quad x\in\mathbb{R}^d
\end{equation}
Our approach solves this iteratively, using the Frank-Wolfe (FW)~\cite{frank_wolfe} update rule:
\begin{align}\label{eqn:toy_fw_step}
x_{t+1} &= (1-\gamma_t)\cdot x_t + \gamma_t\cdot\xi_t\\
\xi_t &=\argmin_x \|x\|^2 \text{ s.t. } \Sigma_{i=1}^d x_i = 1
\end{align}
starting from an arbitrary random feasible point, e.g. sample a random vector and normalize it.
The soft-constraint approach minimizes $\|x\|^2+\lambda \left(\sum_i x_i - 1\right)^2$ using gradient descent (GD), where $\lambda$ is a coefficient that controls the trade-off between the objective function and the soft penalty representing the constraint.

Figure~\ref{fig:fw_vs_gd}, shows the objective value for $d=10$ and the corresponding soft penalty value along the optimization for both FW and GD with several values of $\lambda$.
It can be seend that GD is very sensitive to the trade-off tuned by $\lambda$, often violating the constraint or converging to a sub-optimal objective value. On the contrary, FW converges faster to the optimal solution ($x^*=1/d=0.1$), while strictly satisfying the constraint throughout the optimization.

\subsubsection{Evaluating the Discretizing Projection}\label{sec:projection_effect}
Table~\ref{tbl:projection} evaluates the projection of architectures to the discrete space, as proposed in Section~\ref{sec:projection}. While the commonly used argmax projection violates the constraints by up to 10\%, those are strictly satisfied by our proposed projection.
\vspace{1em}
\begin{table}[htb]
    \centering
    \begin{tabular}{c|c|c|c|c|c|c|}
    Constraint & 35 & 40 & 45 & 50 & 55 & 60 \\ \midrule
    argmax & 36 & 42 & 50 & 54 & 58 & 66 \\ 
    Our Projection & 35 & 40 & 45 & 49 & 54 & 60 \\
    \end{tabular}
    \caption{The effect of discretization method on the actual latency shows that argmax method is harmful.}
    \label{tbl:projection}
\end{table}


%% file: big_table.tex
\begin{table}[t]
\small
    \centering
    \begin{tabular}{c|l|c|c|c|}
    & Model & \makecell[c]{Latency \\ (ms)} & \makecell[c]{Top-1 \\ (\%)} & \makecell[c]{Total  Cost\\(GPU hours)}
    \\ \toprule
         \multirow{14}{*}{\rotatebox[origin=c]{90}{NVIDIA P100 GPU (batch:64)}}
         & MobileNetV3 & 28 & 75.2 & 180N \\ 
         & TFNAS-D & 30 & 74.2 & 236N \\ 
         & \textbf{Ours 25 ms} & \textbf{27} & \textbf{75.7} & 400 + \textbf{15N}\\ \cline{2-5}
         
         & MnasNetA1 & 37 & 75.2 & 40,000N \\
         & MnasNetB1 & 34 & 74.5 & 40,000N \\ 
         & FBNet & 41 & 75.1 & 576N \\ 
         & SPNASNet & 36 & 74.1 & 288 + 408N\\ 
         & TFNAS-B &  44 & 76.3 & 263N \\ 
         & TFNAS-C & 37 & 75.2 & 263N \\ 
         & \textbf{Ours 30 ms} & \textbf{32} & \textbf{77.3} & 400 + \textbf{15N} \\ \cline{2-5} 
         
         & TFNAS-A & 54 & 76.9 & 263N \\ 
         & EfficientNetB0 & 48 & 77.3 &  \\ 
         & MobileNetV2 & 50 & 76.5 & 150N \\ 
         & \textbf{Ours 40 ms} & \textbf{41} & \textbf{77.9} & 400 + \textbf{15N} \\ \midrule
         
         \multirow{15}{*}{\rotatebox[origin=c]{90}{Intel Xeon CPU (batch:1)}}
         
         & MnasNetB1 & 39 & 74.5 & 40,000N\\ 
         & TFNAS-A & 40 & 74.4 & 263N \\ 
         & SPNASNet & 41 & 74.1 & 288 + 408N\\ 
         & OFA CPU\footnotemark & 42 & 75.7 & 1200 + 25N \\ 
         & \textbf{Ours 40 ms} & \textbf{40} & \textbf{75.8} & 400 + \textbf{15N}\\ \cline{2-5}
         
         & MobileNetV3 & 45 & 75.2 & 180N \\ 
         & FBNet & 47 & 75.1 & 576N \\ 
         & MnasNetA1 & 55 & 75.2 & 40,000N \\ 
         & TFNAS-B & 60 & 75.8 & 263N \\ 
         
         & \textbf{Ours 45 ms} & \textbf{44} & \textbf{76.4} & 400 + \textbf{15N}\\ \cline{2-5}
         
         & MobileNetV2 & 70 & 76.5 & 150N \\ 
         & \textbf{Ours 50 ms} & \textbf{50} & \textbf{77.1} & 400 + \textbf{15N} \\ \cline{2-5}
         & EfficientNetB0 & 85 & 77.3 &  \\ 
         & \textbf{Ours 55 ms} & \textbf{55} & \textbf{77.6} &  400 + \textbf{15N}\\ \cline{2-5}
         & FairNAS-C & 60 & 76.7 & 240N \\ 
         & \textbf{Ours 60 ms} & \textbf{61} & \textbf{78.0} & 400 + \textbf{15N}\\ \hline
    \end{tabular}
    \caption{ImageNet top-1 accuracy, latency and cost comparison with other methods. The total cost stands for the search and training cost of N networks.} 
    \label{tab:exp}
\end{table}\footnotetext{\label{fn:ofa}Finetuning a model obtained by 1200 GPU hours.}

%% file: conclusion.tex
\section{Conclusion} \label{Conclusion}
The problem of resource-aware differentiable NAS is formulated as a bilevel optimization problem with hard constraints. Each level of the problem is addressed rigorously for efficiently generating well performing architectures that strictly satisfy the hard resource constraints. HardCoRe-NAS turns to be a fast search method, scalable to many devices and requirements, while the resulted architectures perform better than architectures generated by other state-of-the-art NAS methods. We hope that the proposed methodologies will give rise to more research and applications utilizing constrained search for inducing unique structures over a variety of search spaces and resource specifications.

%% file: supplementary.tex
\begin{center}
    \huge
    Appendix
\end{center}
\appendix

\section{More Specifications of the Search Space}
Inspired by EfficientNet~\cite{effnet} and TF-NAS~\cite{TF-NAS}, we build a layer-wise search space, as explained in Section~\ref{sec:search_space} and depicted in Figure~\ref{fig:super_net} and in Table~\ref{tab:search_space}.
The input shapes and the channel numbers are the same as EfficientNetB0. Similarly to TF-NAS and differently from EfficientNet-B0, we use ReLU in the first three stages. As specified in Section~\ref{sec:micro_search}, the ElasticMBInvRes block is our \textit{elastic} version of the MBInvRes block, introduced in~\cite{mobilenetv2}. Those blocks of stages 3 to 8 are to be searched for, while the rest are fixed.
~
\\
~
\begin{table}[h]
\resizebox{0.45\textwidth}{!}{
\begin{tabular}{c|c|c|c|c|c} \hline
  Stage & Input & Operation & $C_{out}$ & Act & b \\
  \hline
  1  & $224^2\times3$  & $3\times3$ Conv & 32 & ReLU & 1 \\
  2  & $112^3\times32$ & MBInvRes & 16 & ReLU & 1 \\
  3  & $112^2\times16$ & ElasticMBInvRes & 24  & ReLU  & $[2, 4]$ \\
  4  & $56^2\times24$  & ElasticMBInvRes & 40  & Swish & $[2, 4]$ \\
  5  & $28^2\times40$  & ElasticMBInvRes & 80  & Swish & $[2, 4]$ \\
  6  & $14^2\times80$  & ElasticMBInvRes & 112 & Swish & $[2, 4]$ \\
  7  & $14^2\times112$ & ElasticMBInvRes & 192 & Swish & $[2, 4]$ \\
  8  & $7^2\times192$  & ElasticMBInvRes & 960 & Swish & 1 \\
  9  & $7^2\times960$  & $1\times1$ Conv & 1280 & Swish & 1 \\
  10 & $7^2\times1280$ & AvgPool & 1280 & - & 1 \\
  11 & $1280$ & Fc & 1000 & - & 1 \\
  \hline
  \end{tabular}
  }
  \caption{Macro architecture of the one-shot model. "MBInvRes" is the basic block in~\cite{mobilenetv2}. "ElasticMBInvRes" denotes our \textit{elastic} blocks (Section~\ref{sec:micro_search}) to be searched for. "$C_{out}$" stands for the output channels. Act denotes the activation function used in a stage. "b" is the number of blocks in a stage, where $[\underline{b}, \bar{b}]$ is a discrete interval. If necessary, the down-sampling occurs at the first block of a stage.}
  \label{tab:search_space}
  \end{table}

The configurations of the ElasticMBInvRes blocks $c\in\mathcal{C}$ are sorted according to their expected latency as specified in Table~\ref{tab:configurations}.
\begin{table}[H]
\vspace{-1.5mm}
    \centering
    \begin{tabular}{|c||c|c|c|}
    \hline
    c & er & k & se \\
    \hline
    1 & 2 & $3\times 3$ & off \\
    2 & 2 & $5\times 5$ & on \\
    3 & 2 & $3\times 3$ & off \\
    4 & 2 & $5\times 5$ & on \\
    5 & 3 & $3\times 3$ & off \\
    6 & 3 & $5\times 5$ & on \\
    7 & 3 & $3\times 3$ & off \\
    8 & 3 & $5\times 5$ & on \\
    9 & 6 & $3\times 3$ & off \\
    10 & 6 & $5\times 5$ & on \\
    11 & 6 & $3\times 3$ & off \\
    12 & 6 & $5\times 5$ & on \\
    \hline
    \end{tabular}
    \caption{Specifications for each indexed configuration $c\in\mathcal{C}$. "er" stands for the expansion ratio of the point-wise convolutions, "k" stands for the kernel size of the depth-wise separable convolutions and "se" stands for Squeeze-and-Excitation (SE) with $on$ and $off$ denoting with and without SE respectively. The configurations are indexed according to their expected latency.}
    \label{tab:configurations}
\end{table}

\section{Searching for the Expansion Ratio}
Searching for expansion ration ($er$), as specified in Section~\ref{sec:micro_search}, involves the summation of feature maps of different number of channels:
\begin{align}\label{eqn:er_sum}
    y^s_{b, er} = \sum_{er\in\mathcal{A}_{er}} \alpha^s_{b, er}\cdot PWC^s_{b,er}(x^s_b)
\end{align}
where $PWC^s_{b,er}$ is the point-wise convolution of block $b$ in stage $s$ with expansion ratio $er$. 

The summation in \eqref{eqn:er_sum} is made possible by calculating $PWC^s_{b,\bar{er}}$ only once, where $\bar{er}=\max\mathcal{A}_{er}$, and masking its output several times as following:
\begin{align}\label{eqn:er_sum_mask}
    y^s_{b, er} = \sum_{er\in\mathcal{A}_{er}} \alpha^s_{b, er}\cdot PWC^s_{b,\bar{er}}(x^s_b)\odot \mathbbm{1}_{C\leq er\times C_{in}}
\end{align}
where $\odot$ is a point-wise multiplication, $C_{in}$ is the number of channels of $x^s_b$ and the mask tensors $\mathbbm{1}_{C\leq er\times C_{in}}$ are of the same dimensions as of $PWC^s_{b,\bar{er}}(x^s_b)$ with ones for all channels $C$ satisfying $C\leq er\times C_{in}$ and zeros otherwise.

Thus, all of the tensors involved in the summation have the same number of channels, i.e. $\bar{er}\times C_{in}$, while the weights of the point-wise convolutions are shared. Thus we gain the benefits of weight sharing, as specified in Section~\ref{sec:inner_problem}.

\section{Multipath Sampling Code}
We provide a simple PyTorch~\cite{pytorch} implementation for sampling multiple distinctive paths (sub-networks of the one-shot model) for every image in the batch, as specified in Section~\ref{sec:inner_problem}.
The code is presented in figure \ref{code:multipath}.
\input{multipath_code}

\newpage
\section{A Brief Derivation of the FW Step}
Suppose $\mathcal{D}$ is a compact convex set in a vector space and $f \colon \mathcal{D} \to \mathbb{R}$ is a convex, differentiable real-valued function. The Frank-Wolfe algorithm~\cite{frank_wolfe} iteratively solves the optimization problem:
\begin{equation}
    \label{eqn:FW}
    \min_{\mathbf{x} \in \mathcal{D}}f(\mathbf{x}).
\end{equation}
To this end, at iteration $k+1$ it aims at solving:
\begin{equation}
    \label{eqn:FW2}
    \min_{\mathbf{x_k+\bDelta} \in \mathcal{D}}f(\mathbf{x_k+\bDelta}).
\end{equation}
Using a first order taylor expansion of $f$, (\ref{eqn:FW2}) is approximated in the neighborhood of $f(x_k)$, and thus the problem can be written as:
\begin{equation}
    \label{eqn:FW3}
    \min_{\mathbf{x_k+\bDelta} \in \mathcal{D}}\bDelta^T\nabla f(\mathbf{x_k})
\end{equation}
Replacing $\bDelta$ with $\gamma (\bs-\mathbf{x_k})$ for $\gamma\in [0,1]$, problem~\eqref{eqn:FW3} is equivalent to:
\begin{equation}
    \label{eqn:FW4}
    \min_{\mathbf{\mathbf{x_k}+\gamma (\bs-\mathbf{x_k})} \in \mathcal{D}}\gamma (s-\mathbf{x_k})^T\nabla f(\mathbf{x_k}).
\end{equation}
Assuming that $\mathbf{x_k} \in \mathcal{D}$, since $\mathcal{D}$ is convex,  $\mathbf{x_k}+\gamma (\bs-\mathbf{x_k}) \in \mathcal{D}$ holds for all $\gamma\in[0,1]$ iff
$\bs\in \mathcal{D}$. Hence, \eqref{eqn:FW4} can be written as following:
\begin{equation}
    \label{eqn:FW_step}
    \min_{\mathbf{s} \in \mathcal{D}}s^T\nabla f(\mathbf{x_k}).
\end{equation}
Obtaining the minimizer $s_k$ of \eqref{eqn:FW_step} at iteration $k+1$, the FW update step is:
\begin{equation}
\label{eqn:FW_update}
\mathbf{x}_{k+1}\leftarrow \mathbf{x}_k+\gamma(\mathbf{s}_k-\mathbf{x}_k).
\end{equation}
\vspace{-3mm}
\section{Obtaining a Feasible Initial Point}\label{sec:init_point}
Algorithm \ref{alg:BCSFW} requires a feasible initial point $(\vu_0, \bbeta_0) \in \mathcal{S}_{lat}$. assuming such a point exists, i.e. as $t$ is large enough, a trivial initial point $(\vu_0, \bbeta_0):=(\tilde{\vu}, \tilde{\bbeta})$ is associated with the lightest structure in the search space $\mathcal{S}\subset\mathcal{P}_\zeta(\mathcal{S})$, i.e. setting: 
\begin{align}\label{eqn:trivial_init}
    \tilde{\alpha}^s_{b,c}=\mathbbm{1}\{c=1\} \, ; \, \tilde{\beta}^s_b=\mathbbm{1}\{b=2\}
\end{align} for all $s\in\{1,..,S\}, b\in\{1,..,d\},c\in\mathcal{C}$, where $\mathbbm{1}\{\cdot\}$ is the indicator function.
However, starting from this point condemns the gradients with respect to all other structures to be always zero due to the way paths are sampled from the space using the Gumbel-Softmax trick, section \ref{sec:composed_search_space}.

Hence for a balanced propagation of gradients, the closest to uniformly distributed structure in $\mathcal{S}_{LAT}$ is encouraged. For this purpose we solve the following quadratic programs (QP) alternately,
\begin{align}\label{eqn:qp_init}
  &\min_{\vu\in\mathcal{S}^{\tilde{\vu}}} \sum_{s=1}^s \sum_{b=1}^{d-1}\sum_{c=1}^{|\mathcal{C}|-1} (\alpha^s_{b,c+1}-\alpha^s_{b,c})^2 
  \\
  &\min_{\bbeta\in\mathcal{S}^{\tilde{\bbeta}}} \sum_{s=1}^s \sum_{b=1}^{d-1}(\beta^s_{b+1}-\beta^s_{b})^2
  \notag
\end{align}
using a QP solver at each step. 

Sorting the indices of configurations according to their expected latency (see Table~\ref{tab:configurations}), the objectives in (\ref{eqn:qp_init}) promote probabilities of consecutive indices to be close to each other, forming chains of non-zero probability with a balanced distribution up to an infeasible configuration, there a chain of zero probability if formed. Illustrations of the formation of such chains are shown in Figure~\ref{fig:qp_init} for several latency constraints. Preferring as many blocks participating as possible over different configurations, the alternating optimization in \eqref{eqn:qp_init} starts with $\beta$. This yields balanced $\beta$ probabilities as long as the constraint allows it. 
\begin{figure*}[htb]
    \centering
\includegraphics[width=1.\textwidth]{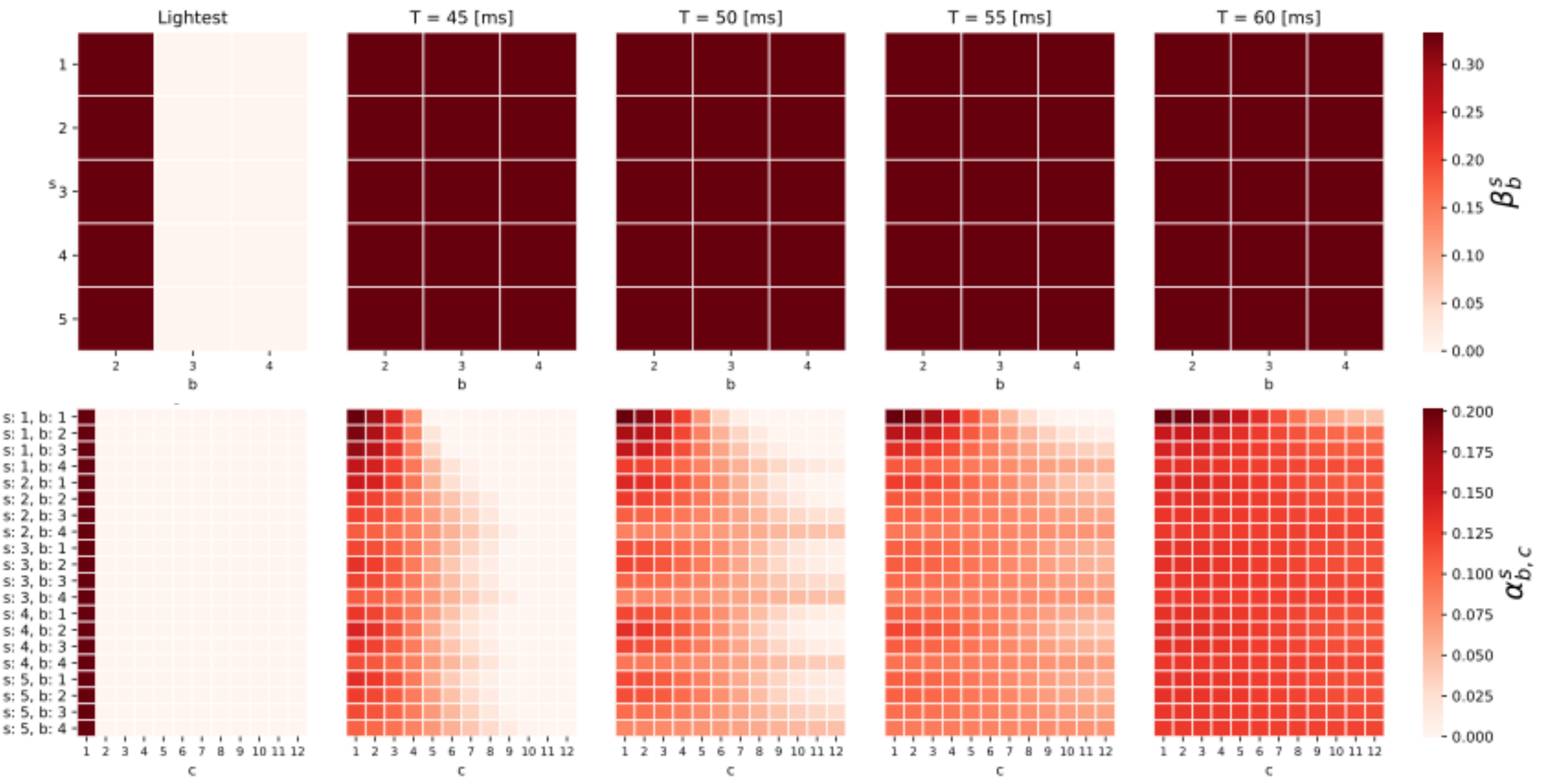}
    \caption{Heat maps representing the initial probability of picking each stage depth (top) and block configuration (bottom), sorted from the lightest to the heaviest (Table~\ref{tab:configurations}). Each couple of frames shows the initial point for a different latency constraint. Rows in each frames stand for different stages (top) and blocks (bottom). The lightest feasible initial point \eqref{eqn:trivial_init} involves only a single configuration of the first block in each stage, avoiding gradients from propagating to others (left). The balanced initial points \eqref{eqn:qp_init} form chains of similar non-zero probability followed by chains of zero probabilities, such that gradients are propagated through feasible paths with a balanced distribution.}
    \label{fig:qp_init}
    \vspace{15mm}
\end{figure*}

The benefits from starting with such initial point are quantified by averaging the relative improvements in top-1 accuracy for several latency constraints $\mathcal{T}=\{35, 40, 45, 50, 55\}$ milliseconds as following:
\begin{align}
\label{eqn:qp_init_vs_lightest}
    \frac{100}{|\mathcal{T}|} \sum_{T\in\mathcal{T}} \frac{Acc^T_{\text{balance init}} - Acc^T_{\text{lighest init}}}{Acc^T_{\text{lighest init}}}
\end{align}
where $Acc^T_{\text{balance init}}$ and $Acc^T_{\text{lighest init}}$ are the top-1 accuracy measured for fine-tuned models generated by searching the space initialized with \eqref{eqn:qp_init} and \eqref{eqn:trivial_init} respectively, under latency constraint $T$.
The calculation in \eqref{eqn:qp_init_vs_lightest} yields $8.3\%$ of relative improvement in favour of \eqref{eqn:qp_init} on average.



\vspace{-3mm}
\section{Proof of Theorem \ref{thm:one_hot_sol}}
In order to proof \ref{thm:one_hot_sol}, we start with proving auxiliary lemmas. To this end we define the \textit{relaxed} Multiple Choice Knapsack Problem (MCKP):

\begin{de}
Given $n\in\mathbb{N}$, and a collection of $k$ distinct covering subsets of $\{1,2,\cdots, n\}$ denoted as $N_i,i\in\{1,2,\cdots,k\}$, such that $\cup_{i=1}^k N_i=\{1,2,\cdots, n\}$ and $\cap_{i=1}^k N_i=\varnothing$ with 
associated values and costs $p_{ij}, t_{ij}~ \forall i \in \{1,\dots,k\}, j \in N_i$ respectively, the relaxed Multiple Choice Knapsack Problem (MCKP) is formulated as following:
\vspace{-3mm}
\begin{align}
\label{eqn:MCKP}
\notag
\max_{vu} &\sum_{i=1}^k \sum_{j\in N_i}  p_{ij} \vu_{ij} \\
\text{subject to}
&\sum_{i=1}^k\sum_{j\in N_i} t_{ij} \vu_{ij} \leq T \\
\notag
&\sum_{j \in N_i}\vu_{ij} = 1 \qquad \forall i \in \{1,\dots,k\} \\
\notag
&\vu_{ij} \geq 0  \qquad \forall i \in \{1,\dots,k\}, j \in N_i
\end{align}
where the binary constraints $\displaystyle{\vu_{ij}\in \{0,1\}}$ of the original MCKP formulation~\cite{MCKP} are replaced with $\displaystyle{\vu_{ij} \geq 0}$.
\end{de}

\begin{de}
 An one-hot vector $\vu_i$ satisfies: $$||\vu_i^*||^0=\sum_{j\in N_i} |\vu_{ij}^*|^0=\sum_{j\in N_i}\mathbbm{1}_{\vu_{ij}^*>0}=1$$ 
where $\mathbbm{1}_A$ is the indicator function that yields $1$ if $A$ holds and $0$ otherwise.
\end{de}

\begin{lem}\label{lem:single_non_one_hot}
The solution $\vu^*$ of the \textit{relaxed} MCKP \eqref{eqn:MCKP} is composed of vectors $\vu_i^*$ that are all one-hot but a single one. 
\end{lem}

\begin{proof}
Suppose that $\vu^*$ is an optimal solution of (\ref{eqn:MCKP}), 
 and two indices $i_1, i_2$ exist such that $\vu_{i_1}^*, \vu_{i_2}^*$ are not one-hot vectors. As a consequence, we show that four indices, $j_1, j_2, j_3, j_4$ exist, such that $\vu_{i_1 j_1}^*, \vu_{i_1 j_2}^*, \vu_{i_2 j_3}^*, \vu_{i_2 j_4}^*\notin\{0,1\}$. 
 
 Define
 \begin{align}\label{eqn:q}
 q&=\frac{t_{i_2 j_2}-t_{i_1 j_1}}{t_{i_2 j_3}-t_{i_2 j_4}} 
 \end{align}
 and
 \begin{align*}
 f&=(t_{i_1 j_1}-t_{i_1 j_2})\left(\frac{p_{i_1 j_1}-p_{i_1 j_2}}{t_{i_1 j_1}-t_{i_1 j_2}}-\frac{p_{i_2 j_3}-p_{i_2 j_4}}{t_{i_2 j_3}-t_{i_2 j_4}}\right)
 \end{align*}
 and assume, without loss of generality, that $f>0$, otherwise one could swap the indices $j_1$ and $j_2$ so that this assumption holds.
 
Set
\begin{align}\label{eqn:Delta}
\Delta= \min\left((1-\vu_{i_1 j_1}^*), \frac{1 - \vu_{i_2 j_3}^*}{|q|}, \vu_{i_1 j_2}^*,  \frac{\vu_{i_2 j_4}^*}{|q|}\right)
\end{align}
such that $\Delta>0$ and construct another feasible solution of \eqref{eqn:MCKP} $\vu_{ij}\leftarrow\vu_{ij}^*$  for all $i,j$ but for the following indices:
\begin{align*}
    \vu_{i_1 j_1} &\leftarrow\vu_{i_1 j_1}^*+\Delta \\
    \vu_{i_1 j_2} &\leftarrow\vu_{i_2 j_2}^*-\Delta \\
    \vu_{i_2 j_3} &\leftarrow\vu_{i_2 j_3}^*+q\Delta \\
    \vu_{i_2 j_4} &\leftarrow\vu_{i_2 j_4}^*-q\Delta
\end{align*}
   
 The feasibility of $\vu$ is easily verified by the definitions in~\eqref{eqn:q} and~\eqref{eqn:Delta}, while the objective varies by:
 \begin{align}
 \notag
 \sum_{i=1}^k& \sum_{j\in N_i}  p_{ij} (\vu_{ij}-\vu_{ij}^*)
 \\\notag
 &= \Delta (p_{i_1 j_1}-p_{i_1 j_2})+q\Delta(p_{i_2 j_3}-p_{i_2 j_4})
 \\\notag
     &=\Delta (t_{i_1 j_1}-t_{i_1 j_2})\left(\frac{p_{i_1 j_1}-p_{i_1 j_2}}{t_{i_1 j_1}-t_{i_1 j_2}}-\frac{p_{i_2 j_3}-p_{i_2 j_4}}{t_{i_2 j_3}-t_{i_2 j_4}}\right)
     \\
     &=\Delta f >0 \label{eqn:contradiction_1}
 \end{align}
 where the last inequality holds due to \eqref{eqn:Delta} together with the assumption $f>0$. Equation~\eqref{eqn:contradiction_1} holds in contradiction to $\vu^*$ being the optimal solution of \eqref{eqn:MCKP}. Hence all the vectors of $\vu^*$ but one are one-hot vectors.
\end{proof}

\begin{lem}\label{lem:two_nonzeros}
The single non one-hot vector of the solution $\vu^*$ of the \textit{relaxed} MCKP \eqref{eqn:MCKP} has at most two nonzero elements.
\end{lem}

\begin{proof}
Suppose that $\vu^*$ is an optimal solution of (\ref{eqn:MCKP}) and an index $\hat{i}$ and three indices $j_1, j_2, j_3$ exist such that $\vu^*_{\hat{i} j_1}, \vu_{\hat{i} j_2}^*, \vu_{\hat{i} j_3}^*\notin\{0,1\}$. 
 
Consider the variables $\bDelta=(\Delta_1,\Delta_2,\Delta_3)^T\in\mathbb{R}^3$ and the following system of equations:
\begin{align}\label{eqn:system}
    t_{\hat{i}j_1}\cdot\Delta_1 + t_{\hat{i}j_2}\cdot\Delta_2 + t_{\hat{i}j_3}\cdot\Delta_3 &= 0\\ \notag
     \Delta_1 + \Delta_2 + \Delta_3 & = 0
%
\end{align}
At least one non-trivial solution $\bDelta^*$ to \eqref{eqn:system} exists, since the system consists of two equations and three variables. 

Assume, without loss of generality, that
\begin{align}\label{eqn:delta_positive}
    p_{\hat{i}j_1}\cdot\Delta_1 + p_{\hat{i}j_2}\cdot\Delta_2 + p_{\hat{i}j_3}\cdot\Delta_3 > 0
\end{align}
Otherwise replace $\bDelta^*$ with $-\bDelta^*$.

Scale $\bDelta^*$ such that
\begin{align}\label{eqn:Delta_scaling}
 0<\vu_{\hat{i} j_1}^*+\Delta_k^*<1 \quad \forall k\in\{1,2,3\}
\end{align}
 and construct another feasible solution of \eqref{eqn:MCKP} $\vu_{ij}\leftarrow\vu_{ij}^*$ for all $i,j$ but for the following indices:
 \begin{align*}
     \vu_{\hat{i} j_1}&\leftarrow\vu_{\hat{i} j_1}^*+\Delta_1 \\
     \vu_{\hat{i} j_2}&\leftarrow\vu_{\hat{i} j_2}^*+\Delta_2 \\
     \vu_{\hat{i} j_3}&\leftarrow\vu_{\hat{i} j_3}^*+\Delta_3
 \end{align*}
 Since $\bDelta^*$ satisfies \eqref{eqn:system} and \eqref{eqn:Delta_scaling}, the feasibility of $\vu$ is easily verified while the objective varies by:
 \begin{align}\label{eqn:contradiction_2}
 \notag
 \sum_{i=1}^k \sum_{j\in N_i}& p_{ij} (\vu_{ij}-\vu_{ij}^*) \notag\\
    &=p_{\hat{i}j_1}\cdot\Delta_1 + p_{\hat{i}j_2}\cdot\Delta_2 + p_{\hat{i}j_3}\cdot\Delta_3 > 0
 \end{align}
where the last inequality is due to \eqref{eqn:delta_positive}. 
 Equation~\eqref{eqn:contradiction_2} holds in contradiction to $\vu^*$ being the optimal solution of \eqref{eqn:MCKP}. Hence the single non one-hot vector of $\vu^*$ has at most two nonzero entries.
\end{proof}
\vspace{-3mm}
In order to prove Theorem~\ref{thm:one_hot_sol}, we use Lemmas~\ref{lem:single_non_one_hot} and~\ref{lem:single_non_one_hot} for $\balpha$ and $\bbeta$ separately, based on the observation that  each problem in~\eqref{eqn:projection_step} forms a \textit{relaxed} MCKP~\eqref{eqn:MCKP}. Thus replacing $\vu$ in~\eqref{eqn:MCKP} with $\balpha$ and $\bbeta$, $p$ is replaced with $\balpha^*$ and $\bbeta^*$ and the elements of \textbf{$t$} are replaced with the elements of $\bbeta^{*T}\Theta^T$ and $\balpha^{*T}\Theta$ respectively.

\vspace{-1.5mm}
\begin{remark}
One can further avoid the two nonzero elements by applying an iterative greedy solver as introduced in~\cite{MCKP}, instead of solving a linear program, with the risk of obtaining a sub-optimal solution.
\end{remark} 

\vspace{-4mm}
\section{2 for 1: $w^*$ Bootstrap - Accuracy vs Cost}
\vspace{-1mm}
In this section we compare the accuracy and total cost for generating trained models in three ways:
\begin{enumerate}
    \vspace{-2mm}
    \item Training from scratch
    \vspace{-2mm}
    \item Fine-tuning $w^*$ for 10 epochs with knowledge distillation from the heaviest model loaded with $\bar{w}^*$.
    \vspace{-2mm}
    \item Fine-tuning $w^*$ for 50 epochs with knowledge distillation from the heaviest model loaded with $\bar{w}^*$.
\end{enumerate}
The last two procedures are specified in Section~\ref{sec:evaluate_inner}.

The results, presented in Figure~\ref{fig:acc_nas2}, show that with a very short fine-tuning procedure of less than 7 GPU hours (10 epochs) as specified in Section~\ref{sec:exp_setting}, in most cases, the resulted accuracy surpasses the accuracy obtained by training from scratch. Networks of higher latency benefit less from the knowledge distillation, hence a longer training is required. A training of 35 GPU hours (50 epochs) results in a significant improvement of the accuracy for most of the models.
\input{accuracy_graph_KD}

\section{Solving the Mathematical Programs Requires a Negligible Computation Time}
In this section we measure the computation time for solving the mathematical programs associated with the initialization point, the LP associated with the FW step and the LP associated with our projection.
We show that the measured times are negligible compared to the computation time attributed to backpropagation.

The average time, measured during the search, for solving the linear programs specified in Algorithm~\ref{alg:BCSFW} and in Section~\ref{sec:projection} and the quadratic program specified in Appendix~\ref{sec:init_point} is $1.15\times 10^{-5}$ CPU hours.

The average time, measured during the search, for a single backpropagation of gradients through the one-shot model is $2.15\times 10^{-3}$ GPU Hours.

The overall cost of solving the mathematical programs for generating $N$ networks is about $0.02N$ CPU hours, which is negligible compared to the overall $400 + 15N$ GPU hours.

%% file: multipath_code.tex
\begin{figure}[htb]
\begin{lstlisting}[language=Python]
import torch
from torch.nn.functional import gumbel_softmax

def multipath(a, ops, x):
    assert C = len(a) == len(ops)
    bs = x.shape[0]
    a = torch.log(a).repeat(bs)
    a = a.reshape(bs, C).transpose(0, 1)
    a_hat = gumbel_softmax(a)

    o = torch.zeros_like(x)
    for ah, op in in zip(a_hat, ops):
        o += ah.view(-1, 1, 1, 1) * op(x)
    
    return o
\end{lstlisting}
\caption{PyTorch Multipath Code}\label{code:multipath}
\end{figure}


%% file: accuracy_graph_KD.tex
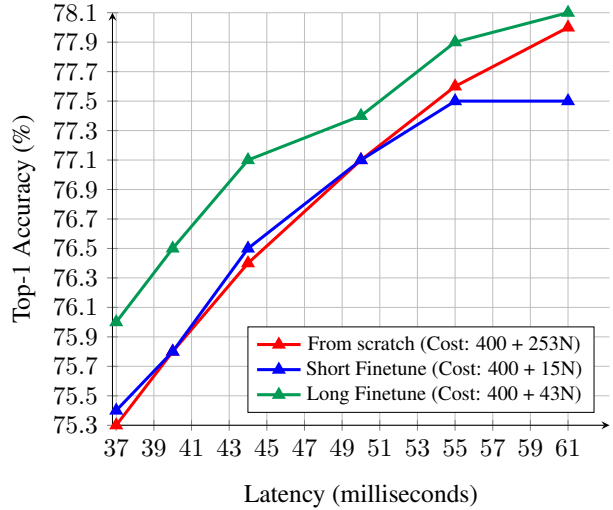
\begin{figure}[H]
\begin{adjustbox}{width=0.47\textwidth}
\begin{tikzpicture}
\begin{axis}[
            axis x line=left,
            axis y line=left,
            xmajorgrids=true,
            ymajorgrids=true,
            grid=both,
            xlabel style={below=1ex},
            enlarge x limits,
            ymin = 75.3,
            ymax = 78.1,
            xmin = 39.0,
            xmax = 61.0,
            xtick = {37,39,...,61},
            ytick = {75.3,75.5,...,78.1},
            ylabel = Top-1 Accuracy (\%),
            xlabel = Latency  (milliseconds),
            legend pos=south east,
            legend style={nodes={scale=0.8, transform shape}},
            legend pos=south east,
    ]


\addplot[color=red, mark=triangle*,very thick]coordinates {(37, 75.3)(40,75.8)(44, 76.4)(50, 77.1)(55, 77.6)(61,78.0)};\addlegendentry{From scratch (Cost: 400 + 253N)}

\addplot[color=blue, mark=triangle*,very thick]coordinates {(37, 75.4)(40,75.8)(44, 76.5)(50, 77.1)(55, 77.5)(61,77.5)};\addlegendentry{Short Finetune (Cost: 400 + 15N)}

\addplot[color=ForestGreen, mark=triangle*,very thick]coordinates {(37, 76.0)(40,76.5)(44, 77.1)(50, 77.4)(55, 77.9)(61,78.1)};\addlegendentry{Long Finetune (Cost: 400 + 43N)}

\end{axis}
\end{tikzpicture}
\end{adjustbox}
\vspace{-5mm}
\caption{Top-1 accuracy on Imagenet vs Latency measured on Intel Xeon CPU for a batch size of 1, for architectures found with our method trained from scratch and fine-tuned from the pretrained super-network}
\label{fig:acc_nas2}
\end{figure}
\vspace{-3mm}